%% file: Arxiv_template.tex
\documentclass[12pt]{article}

\usepackage{natbib}
\usepackage{microtype}
\usepackage{graphicx}
\usepackage{subfigure}
\usepackage{booktabs} 
\usepackage{fullpage}

\usepackage[breaklinks,colorlinks,
            linkcolor=red,
            anchorcolor=blue,
            citecolor=blue
            ]{hyperref}

\usepackage{amsmath}
\usepackage{amssymb}
\usepackage{mathtools}
\usepackage{amsthm}
\usepackage{algorithm}
\usepackage{algorithmic}
\usepackage{enumitem}
\usepackage[capitalize,noabbrev]{cleveref}

\theoremstyle{plain}
\newtheorem{theorem}{Theorem}[section]
\newtheorem{proposition}[theorem]{Proposition}
\newtheorem{lemma}[theorem]{Lemma}
\newtheorem{corollary}[theorem]{Corollary}
\theoremstyle{definition}
\newtheorem{definition}[theorem]{Definition}
\newtheorem{assumption}[theorem]{Assumption}
\newtheorem{remark}[theorem]{Remark}

\usepackage[textsize=tiny]{todonotes}

\usepackage{relsize}
\usepackage[utf8]{inputenc} 
\usepackage[T1]{fontenc}    
\usepackage{url}            
\usepackage{booktabs}       
\usepackage{amsfonts}       
\usepackage{nicefrac}       
\usepackage{microtype}      
\usepackage{smile}
\usepackage{bbm}
\usepackage{xcolor}
\usepackage{bbm}

\input{macro.tex}

\title{\LARGE Estimation of Treatment Effects in Extreme and Unobserved Data}
\author{Jiyuan Tan\thanks{Jiyuan Tan was partially supported by NSF Award IIS-2337916.} \quad 
  Jose Blanchet \quad 
  Vasilis Syrgkanis\thanks{Vasilis Syrgkanis was supported by NSF Award IIS-2337916.} 
  }
\date{}
\begin{document}
\maketitle

\begin{abstract}
Causal effect estimation seeks to determine the impact of an intervention from observational data. However, the existing causal inference literature primarily addresses treatment effects on frequently occurring events. But what if we are interested in estimating the effects of a policy intervention whose benefits, while potentially important, can only be observed and measured in rare yet impactful events, such as extreme climate events? The standard causal inference methodology is not designed for this type of inference since the events of interest may be scarce in the observed data and some degree of extrapolation is necessary. Extreme Value Theory (EVT) provides methodologies for analyzing statistical phenomena in such extreme regimes. We introduce a novel framework for assessing treatment effects in extreme data to capture the causal effect at the occurrence of rare events of interest. In particular, we employ the theory of multivariate regular variation to model extremities. We develop a consistent estimator for extreme treatment effects and present a rigorous non-asymptotic analysis of its performance. We illustrate the performance of our estimator using both synthetic and semi-synthetic data. 
\end{abstract}

\section{Introduction}
We are interested in studying the effect of treatment e.g., different policies and drugs, on rare yet impactful events such as large wildfires, hurricanes, tsunamis and climate change.  These kinds of events happen at an extremely low frequency, but they can cause considerable damage to properties and pose serious threats to people's lives. For instance, we may want to know the effect of more infrastructure investment or other kinds of precautions policies on earthquakes. In many applications – from financial risk to environmental policy – it isn’t enough to know how a treatment changes the average outcome; decision-makers care about whether it alters the extreme tail. More formally, we may want to estimate the effect of treatment $D$ on outcome $Y$ conditioning on some extreme events. Estimating this kind of effect can help policymakers evaluate the impact of a policy and choose the best policy to reduce economic loss and save more lives when disasters happen. 

Despite its clear importance, existing methods fall into two largely disconnected strands, each of which cannot fully address this question. One approach comes from the causal inference literature. Causal inference provide a comprehensive framework for counterfactual reasoning. Causal effect estimation is an important problem in this area, which finds wide applications in healthcare, education, business decision-making, and policy evaluation.  Classic causal inference literature mainly focuses on estimating the average effects among certain groups. Little attention is paid to the causal effect on rare events. The scarcity of extreme data makes inference more challenging than in classic settings. As a result, naively applying classic causal inference estimation methods will produce poor results with large statistical error.  For example, when making policies about earthquakes, we are usually unable to see a strong signal from historical data, as large earthquakes rarely occur and there are fewer samples in the dataset.

On the other hand, the Extreme Value Theory (EVT) studies the tail behaviors for statistical distributions, which provides the ideal tools for analyzing rare events. However, this approach does not take the data structure into consideration. In particular, it does not accommodate counterfactual treatments or adjust for covariates, so it cannot tell us what would happen under an intervention. 

To bridge these gaps, we combine causal inference theory with EVT to provide a novel framework for extreme effect measurement. Following researches in EVT \cite{coles2001introduction}, we use a multivariate regularly varying variable $U$ to model extremity. The rare event can be modeled by the event $\{\|U\| > t\}$ for large $t$. Our proposed estimand can be viewed as the Average Treatment Effect (ATE) conditioning on $\{\|U\| > t\}$ with rescaling as $t$ increases to infinity. Detailed definition and explanation can be found in Section 3. Estimation is challenging because the limiting tail distribution is unknown and must be inferred from finite samples. To improve data efficiency and inference accuracy, we combine tail observations with moderate‐frequency data in an extrapolation scheme, leveraging EVT insights alongside causal‐inference techniques to achieve efficient estimation.

To the best of our knowledge, we are not aware of any work in the literature that considers this problem. In this paper, we take the first step to measure the treatment effect on extreme events. To be more specific, our contributions can be summarized as follows. 

\begin{enumerate}
    \item We propose a measure for the treatment effect on rare events named Normalized Extreme Treatment Effect (NETE), which essentially measures the magnitude of treatment on tailed events.  
    \item We develop two consistent estimators for NETE—a doubly robust (DR) estimator and an inverse propensity weighting (IPW) estimator—by combining recent advances in multivariate tail–dependence estimation \cite{zhang2023wasserstein} with double machine learning methodology \cite{chernozhukov2018double}, and derive finite‐sample, non‐asymptotic error bounds.
    \item Synthetic and semi-synthetic experiments demonstrate a good practical performance of our proposed estimator as compared to baseline estimators adapted from standard causal inference literature. 
\end{enumerate}

\paragraph{\textbf{Related Work}} We briefly review some relevant literature in EVT and causal inference. \cite{coles2001introduction} provides a comprehensive introduction to EVT. A large amount of work focuses on the univariate setting \citet{davisonModelsExceedancesHigh1990,leadbetterBasisPeaksThreshold1991,pickandsiiiStatisticalInferenceUsing1975,smithExtremeValueAnalysis1989}. Recently, there have been many recent works on the multivariate generalization of these results \cite{avella2022kernel,zhang2023wasserstein}. Causal effect estimation is a classical problem in causal inference \citep{rubin1974estimating}.  
Common estimators include IPW \citep{rosenbaum1983central},  
DR methods \citep{bang2005doubly, kang2007demystifying,chernozhukov2016double,chernozhukov2017double,chernozhukov2018double},    
Targeted Maximum Likelihood Estimation (TMLE) \citep{vanderlaan2006targeted}. There have been some efforts in the literature trying to combine the two research areas.  \citet{NadineMaxLinear} considers a special kind of Structural Causal Model (SCM) and shows that the proposed SCM is a kind of max-linear model. They also analyze the asymptotic distribution of their model. \citet{chernozhukovExtremalQuantilesValueRisk2006,chernozhukov2011inference,zhangExtremalQuantileTreatment2018,deuberEstimationInferenceExtremal2024} consider the task of estimating the extreme Quantile Treatment Effect (QTE). The other line of work \citet{gnecco2021causal,mhalla2020causal,bodikCausalityExtremesTime2023} uses EVT to help causal discovery. However, we want to point out that the problems these works consider are quite different from our setting. The most similar setting would be extreme QTE estimation \cite{chernozhukovExtremalQuantilesValueRisk2006,chernozhukov2011inference,zhangExtremalQuantileTreatment2018,deuberEstimationInferenceExtremal2024}, but the QTE still cannot capture on how the expectation of the outcome changes under intervention. 

\section{Preliminary} \label{sec:prelin}

\paragraph{Causal Inference.} We use the potential outcome framework \cite{rubin1974estimating} in this paper. Let $X, D, Y$ be the covariate, binary treatment and outcome, respectively. We denote $Y (d)$ to be the potential outcome when the treatment is set to be $d$ and assume consistency i.e., $Y(D) = Y$ throughout the paper. The Average Treatment Effect (ATE) is defined as \[ \text{ATE } =\mathbb{E} [Y (1) - Y (0)]. \]
The ATE measures the effect of a treatment on the outcome $Y$. In the policy-making example, $D$ is an indicator of whether to use the policy or not. $X$ is a covariate that may influence $D$, like the geographic features of a place, which will influence the local government's decision on policies, and $Y$ can be the economic loss. The ATE in this case provides information about how much loss can be saved if a policy is enforced. Under the following exogeneity and overlap assumptions, the ATE can be identified using the g-formula $\mathbb{E} [\mathbb{E} [Y | X, D = 1 \nobracket] -\mathbb{E} [Y | X, D = 0] \nobracket]$.

\begin{assumption} [Exogeneity] \label{assump:exo}
  \label{assump:unconfouding}The data generation process satisfies $(Y (1), Y(0)) \perp D \mid X$. 
\end{assumption}

Besides, the following overlap assumption is
also often needed for non-asymptotic analysis.
\begin{assumption}
  [Overlap]\label{assump:overlap}There exists constant $c \in (0,1/2)$ such that the estimated propensity $\hat{p} (x) \in [c, 1 - c], \  \forall x \in \mathcal{X}$.
\end{assumption}
This assumption ensures that there is no extremely high or low propensity, which can make estimators unstable. This assumption can be easily achieved by clipping the estimated propensity at some threshold, i.e., setting propensity to be $ \max\{\min\{\hat{p}(X), 1 -c\}, c\} $. 

\paragraph{Extreme Value Theory.} The study of extremity is mainly concerned about the tail behaviors of heavy-tailed distributions, which are often modeled by the regularly varying distributions. In this paper, we modeled extremity by multivariate regularly varying distributions. 

\begin{definition} \label{def:mvarying}
  A random variable $U \in \mathbb{R}^d_+$ is called regularly varying with index $\beta \in (0, \infty)$ if for any norm $\| \cdot \|$ in
  $\mathbb{R}^d$ and positive unit sphere $\mathbb{S}^+ = \{ x \in
  \mathbb{R}^d_+ : \| x \| = 1 \}$, there exists a probability measure $S
  (\cdot)$ on $\mathbb{S}^+$ and a sequence $b_n \rightarrow \infty$ such that $n\tmop{P} ((\| U \| / b_n, U / \| U \|) \in \cdot) \overset{w}{\rightarrow} c \cdot \nu_{\beta} \times S$ for some constant $c > 0$, where $\cdot\times \cdot$ is the product measure and $\nu_{\beta} ([r, \infty)) = r^{-\beta}$ for all $r > 0$.
\end{definition}

The parameter $\gamma = 1/\beta$ is called the Extreme Value Index (EVI), which characterizes the decay rate of the tail. Notice that this definition implies that as $b_n \rightarrow \infty$, the norm of and $\| U \|$ and its angle $U / \| U \|$ become asymptotically independent. We will leverage this fact for estimation in later sections. A typical example of regularly varying distributions is the Pareto distribution.

\begin{definition}
  The density of a Pareto (type II) distribution with index $\beta \in (0, \infty)$ is $f (x) = \beta (1 + x)^{- (\beta + 1)}, \forall x > 0$.
\end{definition}

\cref{def:mvarying} implies that the rescaled norm of a regularly varying variable is asymptotically a Pareto distribution. 



{\tmstrong{Notations.}} In the rest of the paper, we use $\| \cdot \|$ and $\| \cdot \|_1$ as a shorthand for $\ell_1$-norm. We use the asymptotic order notation $o (\cdot), O (\cdot)$ and $\Theta (\cdot)$. We use $\mathbb{E}[\cdot]$ to represent expectation. For a matrix $A$, we denote $A_{\cdot,i}$ to be its $i$-th column. $\text{Unif}([a,b])$ is the uniform distribution on interval $[a,b]$ and $\text{Ber}(p)$ is the Bernoulli distribution with expectation $p$.    

\section{Treatment Effect on Extreme Events}\label{sec:main}
\subsection{Extreme Semi-parametric Inference}

While standard causal estimands capture average effects of $D$ on $Y$, they obscure what happens in the tails—i.e., when rare, high‐impact events occur.  To address this, we model rare events with an explicit noise term $U$. The data we consider is of the form $\{(X_i, D_i, Y_i, U_i)\}_{i = 1}^N$, where $X$, $D$, and $Y$ are as defined in Section \ref{sec:prelin}, and $U$ is an independent extreme noise vector.  We use $\|U\|$ to model the severity of rare events—large norms indicate more extreme realizations.  For example, in a hurricane‐loss application, $U$ might be the vector of maximum wind speed, rainfall, and storm surge; $X$ the region’s location; $D$ the level of infrastructure investment; and $Y$ the resulting economic loss.

In what follows, we introduce a novel estimand that quantifies the causal effect of $D$ on $Y$ specifically in the tail region defined by large $\|U\|$.  We then establish conditions for its identification under multivariate regular variation and propose two consistent estimators. We will make the following i.i.d. assumption. 
\begin{assumption}
  \label{assump:iid}
  The random variables $\{(X_i, D_i, Y_i, U_i)\}_{i = 1}^N$ are i.i.d.. Furthermore, $U$ is regularly varying and is independent of $X, D$. 
\end{assumption}

 We are interested in the effect of treatment on the tail events of $U$. Similar to ATE, a naive definition of the extreme treatment effect would be
\begin{align} \label{eq:ete}
  \theta^{\tmop{ETE}} & = \lim_{t \rightarrow \infty} \mathbb{E}[Y(1) - Y(0) \mid \|U\|> t],
\end{align}
which is simply ATE conditioning on large $ \|U\| $. However, in the case of extreme effects, the outcome may be unbounded due to the presence of extreme noise. As $t$ increases to infinity, this effect may increase to infinity, making this quantity meaningless. Considering the climate change example, it is possible that dramatic climate change will damage or even destroy human societies, causing the effects of some policies to explode even though the policy can effectively reduce losses and slow down the process. Fortunately, regularly varying distributions have the nice property that as $t$ increases to infinity, $\|U\|/ t \mid \|U\|> t$ converges weakly to the Pareto distribution (See \cref{def:mvarying}). Inspired by this property, we can normalize the quantity $ Y(1) - Y(0) \mid \|U\| > t $ by its growth rate. To characterize the growth of this quantity, we introduce the following polynomial growth assumption.
\begin{assumption}[Asymptotic Homogeneous Property]\label{assump:f_homo}
  We assume that the covariate $X$ is bounded, i.e. $\|X\| \leqslant R$. Let $f (X, D, U) =\mathbb{E}[Y \mid X, D, U]$. There exists a $L$-Lipschitz continuous function $g (x,
  d, u)$ and a function $e (t) : \mathbb{R}^+ \rightarrow \mathbb{R}^+$ that satisfies $\lim_{t \rightarrow \infty} e (t) = 0$ and
  \begin{align*}
    | \frac{f (x, d, tu)}{t^{\alpha}} - g (x, d, u) | & \leqslant e (t), \ \forall x \in B_R,\  u \in S^{d - 1}.
  \end{align*}
\end{assumption}
This assumption characterizes the growth of the outcome with respect to the extreme noise.  In many real-world examples, this assumption is satisfied. For instance, research show that landslide volume often follows a power-law relationship with rainfall intensity {\cite{tuganishuri2024prediction}}; the economic loss caused by hurricanes scales polynomially with the maximum wind speed {\cite{zhai2014dependence}}. In these cases, $f$ grows polynomially with respect to $\|U\|$ and $e(t) = 0$ exactly.  We define the Normalized Extreme Treatment Effect (NETE) as
\begin{align}\label{eq:nete}
  \theta^{\tmop{NETE}} & = \lim_{t \rightarrow \infty} \mathbb{E} \left[ \frac{Y(1) - Y(0)}{t^{\alpha}} \mid \|U\|> t \right],
\end{align}
where $\alpha$ is a known index in \cref{assump:f_homo} from prior knowledge. Note that the previous definition (\ref{eq:ete}) is a special case of (\ref{eq:nete}) when $\alpha = 0$. The intuition for the scaling factor $t^\alpha$ is that under \cref{assump:f_homo}, $\mathbb{E}[Y(d)] $ is of the order $ O(\| U \|^\alpha) $ and (\ref{eq:nete}) is of the order $ O(\mathbb{E}[(\|U\|/t)^\alpha \mid \|U\|>t]) $, which is finite if $\alpha < \beta$.  (\ref{eq:nete}) implies that for a large threshold $t$, we have $\mathbb{E}[ Y(1) - Y(0)] \approx t^\alpha \theta^{\tmop{NETE}} $. Therefore, $\theta^{\tmop{NETE}}$ measures the influence of treatment on the susceptibility of outcome with respect to extreme noise $U$. 

We want to remark that NETE naturally sits at the nexus of two well-studied strands of work, tail-conditional expectations in EVT, and average effects or distributional shifts at extreme quantiles, e.g., ATE, CATE and QTE. NETE can be understood as a causal analogue of EVT quantity $ \mathbb{E}[Z/t \mid Z > t] $, where $Z$ is a regularly varying variable. It generalizes ATE to the setting of extreme events and aligns with the growth rate given by EVT.  

\subsection{Extreme Effect Identification and Estimation}

The estimand (\ref{eq:nete}) is designed to measure the treatment effect under extreme events, i.e., extremely large $\|U\|$. In practice, there
may only be a small fraction of extreme samples in the dataset, which creates difficulties for statistical inference. To efficiently estimate the NETE, we leverage the asymptotic independence property of regularly varying variables
(See  Definition~\ref{def:mvarying}) to derive a novel identification formula.
In particular, we have the following decomposition.
\begin{align}
  \lim_{t \rightarrow \infty} \mathbb{E} \left[ \frac{Y (1) - Y
  (0)}{t^{\alpha}} \mid \|U\|> t \right] & = \lim_{t \rightarrow \infty}
  \mathbb{E} \left[ \frac{f (X, 1, U) - f (X, 0, U)}{t^{\alpha}} \mid \|U\|> t
  \right]  \label{eq:naive_id}\\
  & = \lim_{t \rightarrow \infty} \mathbb{E} \left[ \frac{f (X, 1, U) - f (X,
  0, U)}{\|U\|^{\alpha}} \cdot \left( \frac{\|U\|}{t} \right)^{\alpha} \mid
  \|U\|> t \right] \nonumber\\
  & = \lim_{t \rightarrow \infty} \mathbb{E} \left[ g (X, 1, U /\|U\|) - g
  (X, 0, U /\|U\|) \cdot \left( \frac{\|U\|}{t} \right)^{\alpha} \mid \|U\|> t
  \right]  \label{eq:naive_id2}
\end{align}

where we use Assumption~\ref{assump:f_homo} in the third equality. We can prove that the above quantity equals to
\[ \lim_{t \rightarrow \infty} \mathbb{E}[g (X, 1, U /\|U\|) - g (X, 0, U
   /\|U\|) \mid \|U\|> t] \cdot \lim_{t \rightarrow \infty}
   \mathbb{E}[\|U\|^{\alpha} / t^{\alpha} \mid \|U\|> t] . \]
The first factor measures the average effect of treatment across different
directions, while the second factor only depends on the norm of the extreme
noise, which can be estimated via standard techniques in extreme value theory.
We summarize the identification formula in the following proposition.

\begin{proposition}
  [Identification]\label{prop:id}Suppose that $U$
  is multivariate regularly varying and Assumption~\ref{assump:exo}, \ref{assump:iid}, \ref{assump:f_homo}
  hold, we have
  \begin{align*}
    \theta^{\tmop{NETE}} = \lim_{t \rightarrow \infty} \mathbb{E}[g (X, 1, U
    /\|U\|) - g (X, 0, U /\|U\|) \mid \|U\|> t] \cdot \lim_{t \rightarrow
    \infty} \mathbb{E}[\|U\|^{\alpha} / t^{\alpha} \mid \|U\|> t] .
  \end{align*}
\end{proposition}

Proposition \ref{prop:id} separates the estimation of NETE into two parts, the expectation of the spectral measure and the index estimation, which facilitates the estimation. While in theory naive identification (\ref{eq:naive_id}) works as well, we found that in practice (\ref{eq:naive_id}) performs poorly (See \cref{sec:exp} for empirical experiments). One reason is that without properly scaling, the (\ref{eq:naive_id}) suffers from exploding $\|U\|$, causing larger estimation errors. 


\begin{algorithm}[!tb]
  \caption{Algorithm for NETE Estimation}\label{alg:nete}
  \begin{algorithmic}[1]
    \REQUIRE Dataset $\mathcal{D}=\{(X_i,D_i,Y_i,U_i)\}_{i=1}^n$, threshold $t$, exponent estimation $ \hat{\alpha}_{n} $, estimator
    \STATE Randomly split $\mathcal{D}$ into two equal parts $\mathcal{D}_1$ and $\mathcal{D}_2$
    \STATE Using $\mathcal{D}_1$, estimate:
    \begin{enumerate}[label=\alph*.]
      \item Propensity function $\hat p(x)$ via regression of $D$ on $X$
      \item Pseudo‐outcome regression $\hat g(x,d,s)$ by regressing $Y/\|U\|^{\hat{\alpha}_{n}}$ on $(X,D,U/\|U\|)$
    \end{enumerate}
    \STATE Define index set $\mathcal{I}=\{i:\|U_i\|>t, (X_i,D_i,Y_i,U_i)\in\mathcal{D}_2\}$ and set $S_i=U_i/\|U_i\|$ for $i\in\mathcal{I}$
    
      \IF{estimator = IPW}
        \STATE Compute
        \begin{equation}
            \hat\eta_{n,t}^{\mathrm{IPW}}
          = \frac{1}{|\mathcal{I}|} \sum_{i\in\mathcal{I}} \frac{Y_i}{\|U_i\|^{\hat{\alpha}_n} } \Bigl(\frac{D_i}{\hat p(X_i)} - \frac{1-D_i}{1-\hat p(X_i)}\Bigr).\label{eq:eta_ipw}
        \end{equation}
      \ELSE
        \STATE Compute
        \begin{equation}
            \hat\eta_{n,t}^{\mathrm{DR}} = \frac{1}{|\mathcal{I}|} \sum_{i\in\mathcal{I}}\Bigl[\hat g(X_i,1,S_i)-\hat g(X_i,0,S_i)
          + \frac{D_i-\hat p(X_i)}{\hat p(X_i)(1-\hat p(X_i))}\bigl(Y_i/\|U_i\|^{\hat{\alpha}_{n}}-\hat g(X_i,D_i,S_i)\bigr)\Bigr]. \label{eq:eta_dr}
        \end{equation}
      \ENDIF
      \STATE Compute adaptive Hill estimator on $\{\|U_i\|:i\in\mathcal{I}\}$:
      \begin{equation}\label{eq:mu}
          \hat\gamma_n = \frac{1}{k} \sum_{j=1}^{k} \log\frac{\|U_{(j)}\|}{\|U_{(j+1)}\|}, \quad
        \hat\mu_n = \frac{1}{1-\hat{\alpha}_{n}\hat\gamma_n},
      \end{equation}
      where $\|U_{(1)}\|\ge\dots\ge\|U_{(k+1)}\|$ and $k$ is chosen by 
      \begin{equation}
           k  = \max \left\{ k \in \{l_n, \cdots, n\} \  \text{and} \  \forall i \in \{l_n,\cdots, n\}, | \hat{\gamma}_i - \hat{\gamma}_k | \leqslant \frac{\hat{\gamma}_i r_n (\delta)}{\sqrt{i}} \right\}, \notag
      \end{equation}
    \textbf{Return}:  $\widehat{\theta}_{n,t}^{\mathrm{estimator}} =  \hat\eta_{n,t}^{\mathrm{estimator}}\cdot\hat\mu_n.$
  \end{algorithmic}
\end{algorithm}

Inspired by this decomposition, we estimate the two factors separately. We summarize our estimators in Algorithm \ref{alg:nete}. To make our framework more flexible, we allow an approximate scaling exponential $\hat{\alpha}_n$ as input in \cref{alg:nete}. $\hat{\alpha}_n$ can be obtained from some prior knowledge or via other heuristics. For the first factor, we design two estimators, the Inverse Propensity Weighting (IPW) and the Doubly Robust (DR) estimators. To derive the estimators, we first randomly split the data into equal halves and use the first half for nuisance estimation, i.e., propensity and outcome. We use the first half of data to regress ($X, D, U /\|U\|$) on $Y/\|U\|^{\hat{\alpha}_n}$ to get (normalized) pseudo-outcome $\hat{g}$ and regress $X$ on $D$ to get an estimation of the propensity function $\hat{p}$. Then, we use
the second half for estimation. The IPW and DR estimators are defined in (\ref{eq:eta_ipw}) and (\ref{eq:eta_dr}), respectively.

Notice that the second factor is the $\alpha$-moment of the random variable $\|U\|/ t \mid \|U\|> t$, which converges weakly to a Pareto distribution as $t$ increases to infinity. Therefore, this quantity equals to the $\alpha$ moment of a standard Pareto $ 1/(1-\alpha \gamma) $ and the problem can be reduced to estimating the EVI of an asymptotic Pareto distribution. Here, we use the adaptive Hill estimator in (\ref{eq:mu}) from {\cite{boucheron2015tail}}, which provide a data-driven method for choosing the threshold.  Putting the two estimations together, we get our estimator of the NETE $\hat{\theta}_{n, t}^{\cdot} = \hat{\eta}_{n, t}^{\cdot}
\cdot \hat{\mu}_n$, where the superscript ${\cdot}$ can be DR or IPW.

\subsection{Non-asymptotic Analysis}

Up to now we have worked under very
mild regular variation and asymptotic homogeneity conditions, which suffice to prove the consistency of our two‐step estimator in the limit $n, t \rightarrow
\infty$. However, to obtain non--asymptotic, finite‐sample deviation bounds for both the spectral‐measure term and the tail‐index term, we must invoke a more structured tail model.  In particular, existing results such as those in {\cite{zhang2023wasserstein}} rely on the fact that, beyond regular variation, the noise vector behaves exactly like a (possibly linearly transformed) Pareto distribution. Although this is admittedly stronger than
mere second‐order regular variation, it is at present the only framework in which we can directly apply sharp concentration inequalities and Wasserstein‐distance bounds for spectral‐measure estimation. We therefore make the following Pareto‐type assumption.

\begin{assumption}
  \label{assump:mvary}We assume that the distribution of $U$ comes from the
  following class of models
  \begin{align*}
    M & = \cup_{k = 1}^{\infty} M_k,
  \end{align*}
  
  where \ $M_k =\{\mathcal{L}(X) : U = AZ$, for $A \in \mathcal{A}$ and
  $\nobracket \mathcal{L}(Z) \in \tilde{M}_k \}$. The set of possible
  distributions for the components $Z$ is
  \[ \tilde{M}_k = \left\{ \begin{array}{c}
       Z \ \text{admits a (Lebesgue) density } h (z)  \  \text{in } \mathbb{R}_+^{d_z} \\
       \left| \frac{h (z) - \beta^m  \prod_{i = 1}^m (1 +
       z_i)^{- (\beta + 1)}}{\beta^m  \prod_{i = 1}^{d_z} (1 + z_i)^{- (\beta +
       1)}} \right| \leqslant \xi k^{- s}, \forall z\\
       h (z) \varpropto \prod_{i = 1}^m (1 + z_i)^{- (\beta + 1)}  \text{if} \ \|z\|_1 > \zeta k^{\frac{1 - 2 s}{\beta}}
     \end{array} \right\}, \]
  and the set of possible matrices $\mathcal{A}$ is
  \[ \mathcal{A}= \left\{ A \in \mathbb{R}_+^{d_u \times d_z} : l \leq \min_i
     \|A_{\cdot i} \|_1 \leq \max_i \|A_{\cdot i} \|_1 \leq u \ \text{and} \  JA
     \geq \sigma \right\} ,\]
  where ${J}A = \sqrt{\det(A^\mathsf{T}A)}$. Throughout, we assume the constants satisfy $m \geq d \geq 2, 0 < l < 1 < u,
  0 < s < 1 / 2, \sigma > 0$, $0 < \xi < 1$, and $\zeta > 0$.
\end{assumption}

This assumption states that the extreme variable is a linear transformation of an approximate Pareto distribution. The parameter $s$ measures how close $Z$ is to a standard multivariate Pareto distribution. A small $s$ means the distribution is far from Pareto. With these assumptions, we are ready to state our main theorem, which give a
non-asymptotic rate to our estimand.

\begin{theorem}
  \label{thm:rate}Suppose that Assumption
  \ref{assump:exo}, \ref{assump:overlap}, \ref{assump:iid}, \ref{assump:f_homo}, \ref{assump:mvary} hold, $\alpha < \beta$, where $\alpha$
  and $\beta$ are defined in \cref{assump:f_homo} and \cref{assump:mvary}
  respectively. Furthermore, for any fixed $t$, with probability at least $1 -
  \delta$,
  \begin{align*}
    |p (X) - \hat{p} (X) | \leqslant R_p (n, \delta) & ,  | \hat{\alpha}_{n} - \alpha| \leqslant R_\alpha(n,\delta),\\
    |\mathbb{E}[Y/\|U\|^{\alpha} \mid X, D, U /\|U\|, \|U\|> t] &- \hat{g} (X, D, U /\|U\|)
    | \leqslant R_g (n_t, \delta),
  \end{align*}
  
  where $n_t = \sum_{i = 1}^{n/2} I (\|U_i \|> t)$ and $R_p, R_g, R_\alpha$ are estimation
  errors that are monotonically decreasing with respect to sample size. Then, with probability at
  least $1 - \delta, \delta \in (0, 1 / 2)$, we have
  \begin{align}
    \left| {{\widehat{\theta }_{n, t}^{\text{DR} }} }   -
    \theta^{NETE} \right|  & \leqslant O (\sqrt{R_p (n/2, \delta)R_g ({n_t}, \delta)} + t^{\beta / 2} n^{-1 / 2}+ \log (1 / \delta) n^{- 1 / (2 +\beta)}  \notag\\
    & \quad\quad \quad \quad + t^{- \min \{1, \beta\}}  + t^{- \beta s / (1 - 2 s)} + \log(t)R_\alpha(n,\delta)+ e (t)).\label{eq:rate_dr} 
  \end{align}
  and
  \begin{align}
    \left| {{\widehat{\theta }_{n, t}^{\text{IPW}}} }   - \theta^{NETE}
    \right| & \leqslant O (R_p (n/2, \delta)  +t^{\beta / 2} n^{- 1 / 2} + \log (1 / \delta)
    n^{- 1 / (2 + \beta)} \notag\\
    &\quad\quad \quad \quad  + t^{- \min \{1, \beta\}} + t^{- \beta s / (1 - 2 s)}+ \log(t)R_\alpha(n,\delta) + e (t)). \label{eq:rate_ipw}
  \end{align}
\end{theorem} 

The error bound (\ref{eq:eta_dr}) consist of the nuisance error $ \sqrt{R_p (n/2, \delta)R_g ({n_t}, \delta)}$, variance $ t^{\beta / 2} n^{-1 / 2} $, EVI estimation error $ \log (1 / \delta) n^{- 1 / (2 +\beta)} $, $\alpha$ error $ R_\alpha (n,\delta)$ and bias terms $ t^{- \min \{1, \beta\}}  + t^{- \beta s / (1 - 2 s)} +  e (t) $. Similar pattern holds for (\ref{eq:rate_ipw}). Given this general result, we choose the threshold $t$ in a data-driven way to obtain a better rate. The idea is to use the estimated index to balance
the bias and variance terms in (\ref{eq:rate_dr}) and (\ref{eq:rate_ipw}). The following corollary gives
the convergence rate in two different regimes.

\begin{corollary}\label{cor:rate}
Under the assumptions of Theorem~\ref{thm:rate}, further
  suppose that
  \begin{align*}
    R_p (n, \delta) = \Theta (\log (1 / \delta) n^{- 1 / 2}) , R_g (n, \delta) = \Theta (\log (1 / \delta) n^{- 1 / 2}), R_\alpha (n, \delta) = \Theta( \log(1/\delta) n^{-c_\alpha}),
  \end{align*}
  for some $c_\alpha > 0 $, the following conclusions hold.
  \begin{enumerate}
    \item If $s \in (0, 1 / (2 + \max \{1, \beta\}))$, takes $t_n = \Theta (n^{
    (1 - 2 s) \hat{\gamma}_{n}})$, with probability at least $1 - \delta$,
    we have
    \begin{align*}
      | {{\widehat{\theta }_{n, t}^{\text{DR} }} }   - \theta^{NETE} | & = O
      (e (t_n) + n^{- s} \log (1 / \delta) +  n^{-c_\alpha} \log(n)\log (1 / \delta)) .
    \end{align*}
    \item If $s \in [1 / (2 + \max \{1, \beta\}), 1 / 2)$, \ takes $t = \Theta
    (n^{(\hat{\gamma}_{n} / (1 + 2 \min \{1, \hat{\gamma}_{n} \})})$,
    with probability at least $1 - \delta$, we have
    \begin{align*}
      | {{\widehat{\theta }_{n, t}^{\text{DR} }} }   - \theta^{NETE} | & = O
      (e (t_n) + n^{- 1 / (2 + \max \{\beta, 1\})} \log (1 / \delta) +  n^{-c_\alpha} \log(n)\log (1 / \delta)) .
    \end{align*}
  \end{enumerate}
\end{corollary}

Similar results hold for the IPW estimator. Due to limited space, we leave the result for IPW in the appendix. Many common machine learning algorithms, e.g., Lasso, logistic regression, neural networks, can achieve $O(n^{-1/2})$ rate in the assumption of \cref{cor:rate}.  We want to highlight that if $e(t)$ decays fast enough and become negligible compared to the other term and we know the correct scaling exponential $\alpha$, Corollary~\ref{cor:rate} matches the rate of {\cite[Theorem 3.1]{zhang2023wasserstein}} without prior knowledge on the index $\beta$ in the Assumption~\ref{assump:mvary}. Besides, if we have additional prior knowledge on $e(t)$ and $c_\alpha$, we can adjust the choice of threshold $t$ to achieve a better rate. 

\begin{remark}
  When the extreme noise is 1-dimensional, the spectral measure is trivially $\delta_{\{1\}}$ and there is no need to estimate the spectral measure. Following a similar argument of  Theorem~\ref{thm:rate} and 
  Corollary~\ref{cor:rate}, we can obtain a convergence rate of $O (e
  (t_n) + \log(1/\delta)n^{- 1 / (2 + \beta)} + \log(1/\delta) n^{-c_\alpha} )$.
\end{remark}


\begin{remark}
  Assumption \ref{assump:mvary} may seem restricted at first glance. This assumption is used here because the non-asymptotic result for regularly varying extreme distributions is rare in the literature and the goal of this paper is not to develop a new estimator for the spectral measure. To the best of our knowledge, {\cite{zhang2023wasserstein}} is the only paper that gives such a result under Assumption \ref{assump:mvary}. In fact, Assumption \ref{assump:mvary} can be easily replaced by the following two assumptions in our proof. (1) The extreme noise $U$ is regularly varying and its norm $\|U\|$ satisfies the von Mises condition in {\cite{boucheron2015tail}}. (2) There exists an upper bound for the bias term $ \| \mathbb{E} [f(U/\|U\|) \mid \|U\|>t] - \lim_{t\rightarrow\infty}\mathbb{E} [f(U/\|U\|) \mid \|U\|>t]\leqslant O(t^{- c_0})$, for some constant $c_0 > 0$ for a fixed Lipschitz function $f$. We leave this generalization to future work.
\end{remark}

\section{Experiments}\label{sec:exp}
Having established in Section \ref{sec:main} that under our regularity and overlap assumptions the DR‐ and IPW‐based extreme treatment estimators enjoy a provable non‑asymptotic error bound, we next evaluate their finite‑sample behavior and compare with our estimators with naive estimators that does not consider the regularly varying structure. In what follows, Section \ref{subsec:syn} presents purely synthetic simulations with known NETE. Section \ref{subsec:semisyn} then moves to a semi‑synthetic setting—using real noise from wavesurge datasets—to assess practical performance under realistic complexities.

\subsection{Synthetic Dataset}\label{subsec:syn}
The data generation process we use in this subsection is 
\begin{align*}
    &X \sim \text{Unif}([0,1]^5), D \sim \text{Ber}(p(X)), \text{where} \  p(x) = 1/(1 + \exp(-x^\mathsf{T} b)), \\
    & Y = \| U \|^\alpha(D + U/\|U\| + \epsilon) + \| U \|^{\alpha/2}, \epsilon\sim\text{Unif}(-1,1), 
\end{align*}
where $\alpha > 0$ is a constant and $b\sim N(0,1), A \sim \text{Unif}([1,2]^{d_u\times d_z})$. We consider two ways of generating the extreme noise. The first one follows \cref{assump:mvary}. 
\begin{align*}
     Z = (Z_1,\cdots, \ Z_{d_z}), \ Z_i \sim \text{Pareto}(\beta),\ U = AZ, A \in \mathbb{R}^{d_u\times d_z}.
\end{align*}
We also consider a Pareto mixture, i.e., $ U = (U_1,\cdots, U_{d_u}), U_i \sim 0.5 \text{Pareto} (\beta) + 0.5 \text{Pareto}(\beta + 1) $.  Note that \cref{assump:f_homo} is satisfied with $ e(t) = t^{-\alpha/2} $. By \cref{prop:id}, we can calculate the ground-truth effect. The graph below shows the Mean Square Error (MSE) with our estimator using different sample sizes. We take different values for $\alpha, \beta$ in the experiments. In this case, by \cref{prop:id}, we know that the ground-truth NETE is $ {1}/{(1-\alpha/\beta)} $. We use Mean-Square-Error (MSE), $\mathbb{E}[(\hat{\theta} - \theta^{\text{NETE}})^2]$, to measure the error. As a baseline, we compare our estimator with naive IPW and DR estimators. Naïve‑IPW simply applies the standard IPW estimator to the $U_i$ that has norm larger than a threshold $t$, ignoring any tail‐index modeling. Similarly, Naïve‑DR augments it with the usual doubly‑robust correction term but likewise ignores the Pareto structure. We leave the detailed math formulation of the baseline estimators to the appendix. The thresholds rule in \cref{cor:rate} is used in the experiments and we use the same threshold selection rules for all estimators. We estimate the scaling exponential $\alpha$ by doing linear regression $ \log (|Y|) \sim \log (\|U\|)$ and use the coefficient of $\log(\|U\|)$ as $\hat{\alpha}_n$. We leave the experiment details to the appendix. \cref{fig:synthetic} and \cref{fig:synthetic2} show the experiment results. In the following, we use EVT-IPW and EVT-DR to represent $ \widehat{\theta}_{n,t}^{\mathrm{IPW}} $ and  $ \widehat{\theta}_{n,t}^{\mathrm{DR}} $ in \cref{alg:nete}.

\begin{figure}[!htbp]
  \centering
  \begin{minipage}[b]{0.45\textwidth}
    \centering
    \includegraphics[width=\linewidth]{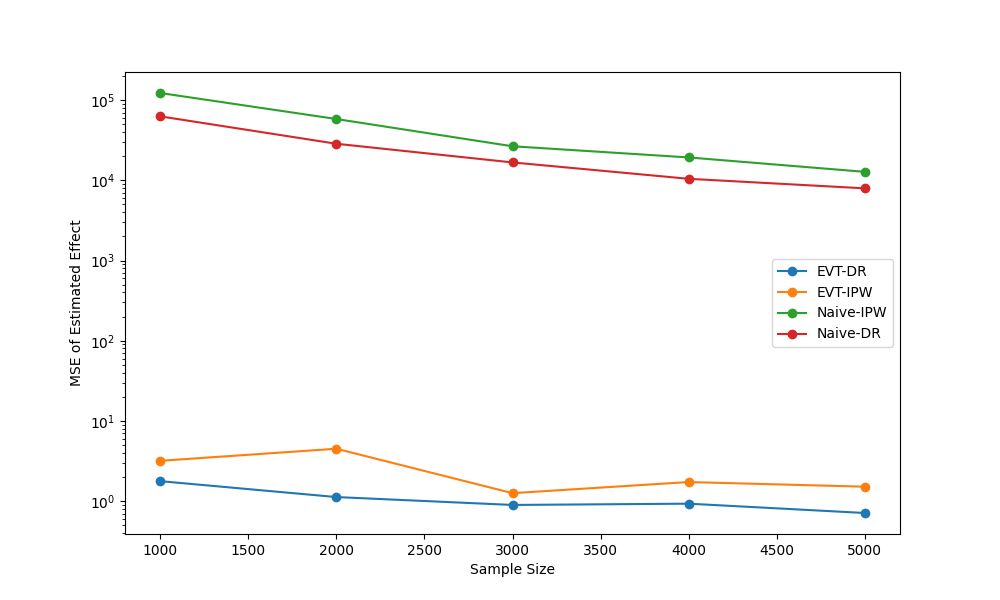}
  \end{minipage}
  \hfill
  \begin{minipage}[b]{0.45\textwidth}
    \centering
    \includegraphics[width=\linewidth]{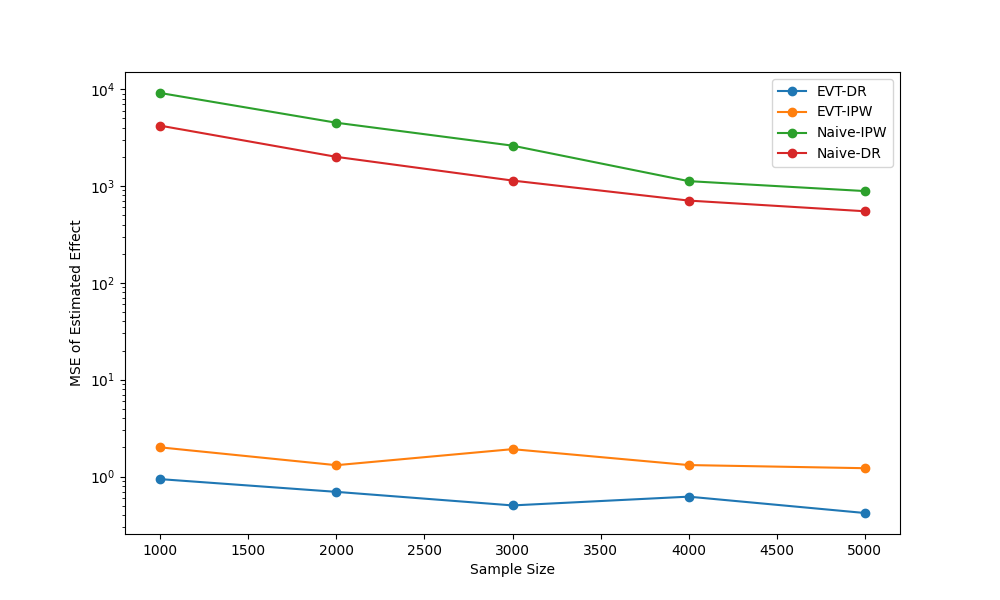}
  \end{minipage}
  \vspace{1em} 
  \begin{minipage}[b]{0.45\textwidth}
    \centering
    \includegraphics[width=\linewidth]{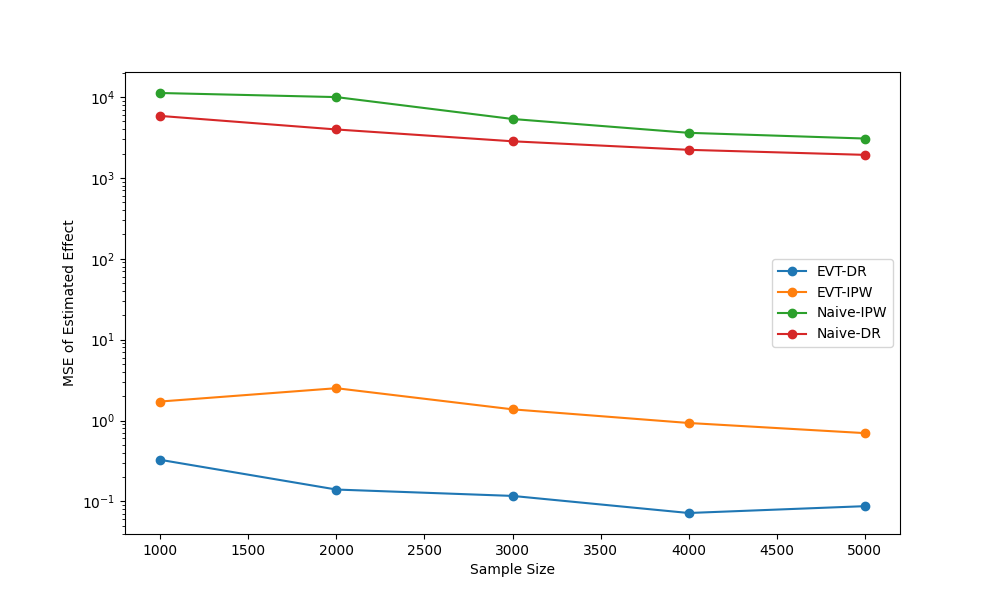}
  \end{minipage}
  \hfill
  \begin{minipage}[b]{0.45\textwidth}
    \centering
    \includegraphics[width=\linewidth]{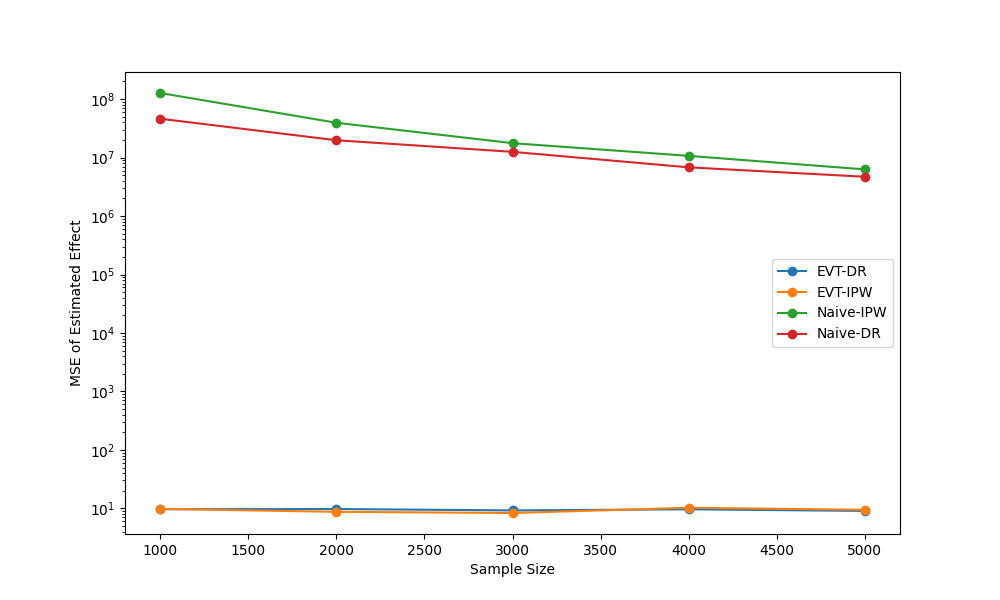}
  \end{minipage}

  \caption{Experiment results of four different configurations when the extreme noise is a linear transformation of Pareto variables. The configures of upper left, upper right, lower left and lower right are $ \alpha, \beta, d_z, d_u = (1, 1.5, 50, 10), (1, 1.5, 30, 5), (1, 2.5, 30, 5)$ and $ (2, 2.5, 30, 5)$ respectively. The results are averages of 50 repeated experiments. We use EVT-IPW and EVT-DR to represent $ \widehat{\theta}_{n,t}^{\mathrm{IPW}} $ and  $ \widehat{\theta}_{n,t}^{\mathrm{DR}} $ in \cref{alg:nete}.}
  \label{fig:synthetic}
\end{figure}

 Figure \ref{fig:synthetic} and \cref{fig:synthetic2} show that under different configurations of $\alpha,\beta, d_u, d_z$, our estimators generally perform better than the baseline estimators. The reason is that our estimators can make better use of the regularly varying structure. In general, EVT-DR achieves the smallest MSE in most experiments and is robust under different configurations. Note that the Pareto mixture does not satisfy \cref{assump:mvary}. \cref{fig:synthetic2} shows that our method still maintain a good performance even if \cref{assump:mvary} is violated. We also observe that sometimes the MSE increases with more samples in \cref{fig:synthetic2}. An explanation for this is that violation of \cref{assump:mvary} causes the threshold selection rule in \cref{cor:rate} not to be applicable and the variance term dominates the error.  
 

\begin{figure}[!htbp]
  \centering
  \begin{minipage}[b]{0.45\textwidth}
    \centering
    \includegraphics[width=\linewidth]{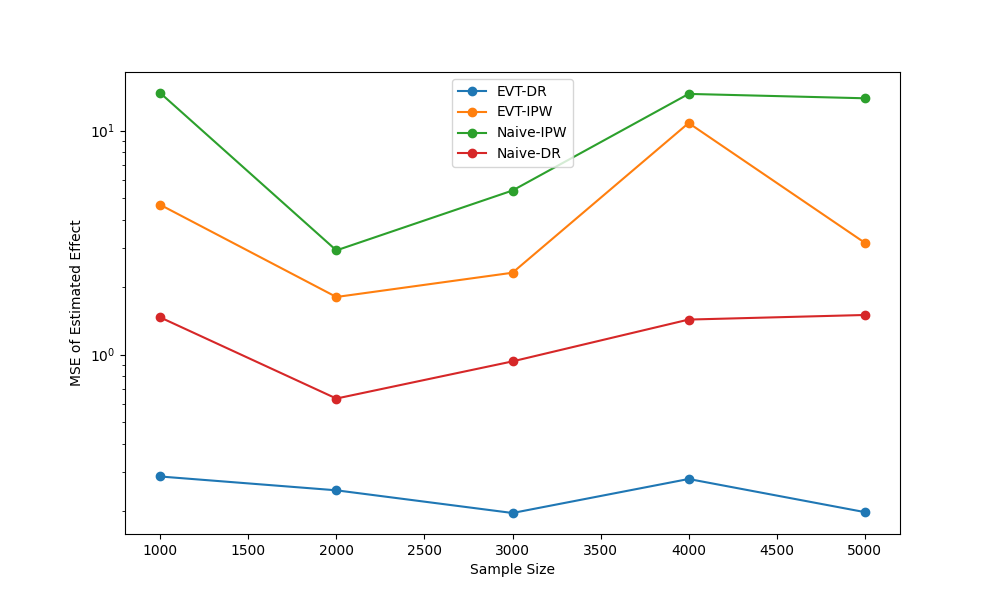}
  \end{minipage}
  \hfill
  \begin{minipage}[b]{0.45\textwidth}
    \centering
    \includegraphics[width=\linewidth]{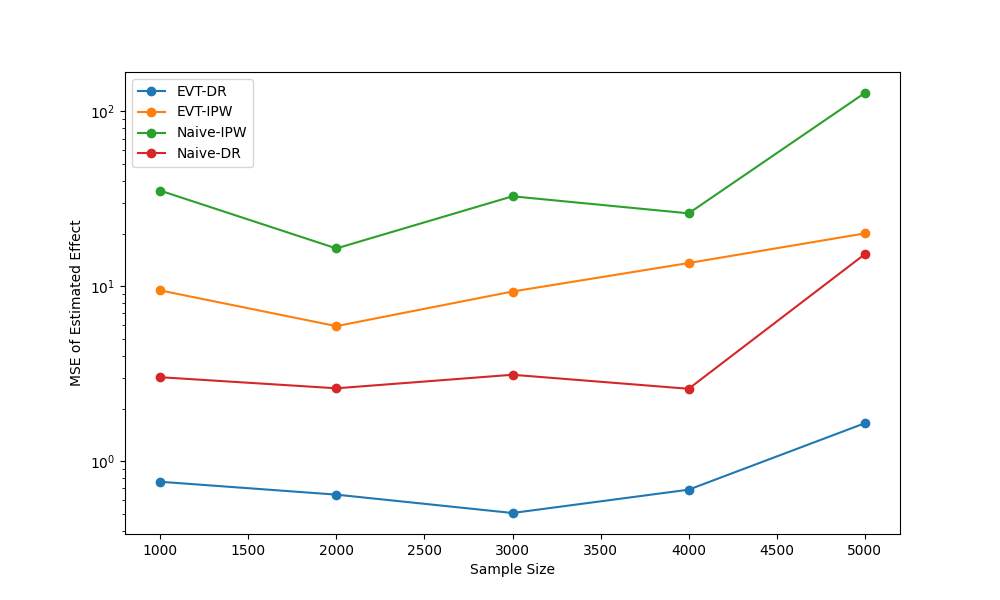}
  \end{minipage}
  \vspace{1em} 
  \begin{minipage}[b]{0.45\textwidth}
    \centering
    \includegraphics[width=\linewidth]{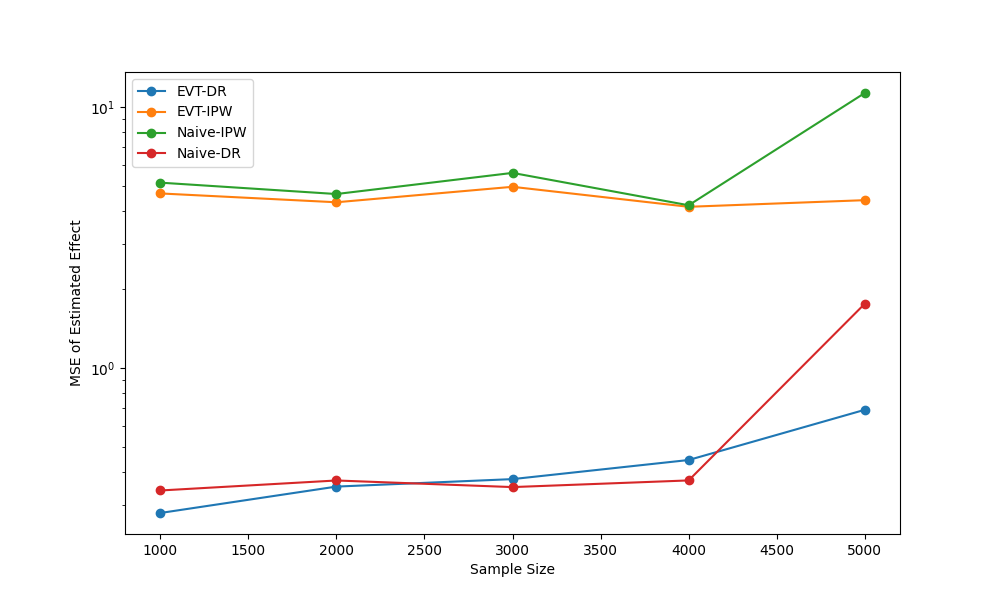}
  \end{minipage}
  \hfill
  \begin{minipage}[b]{0.45\textwidth}
    \centering
    \includegraphics[width=\linewidth]{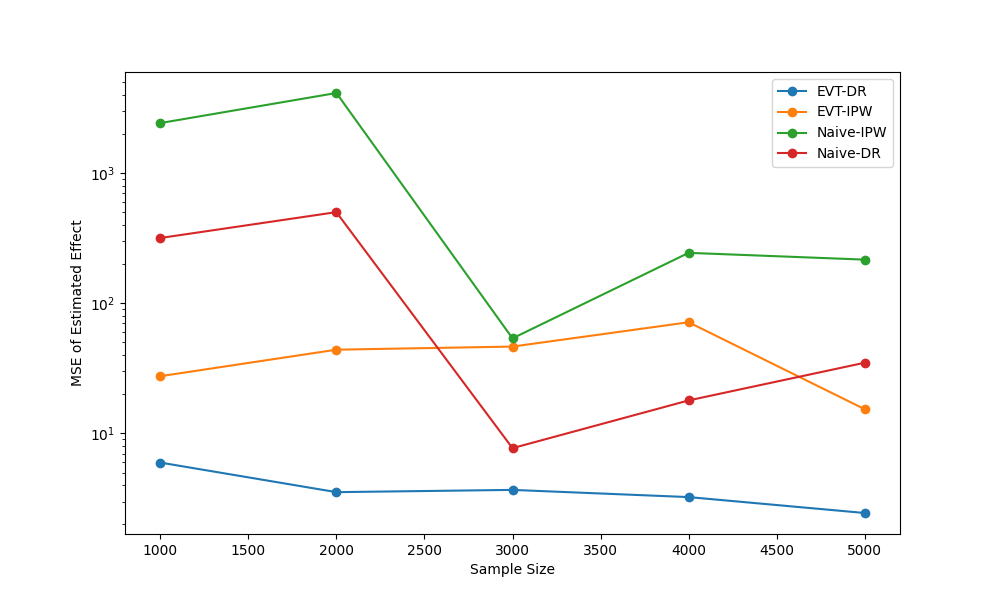}
  \end{minipage}

  \caption{Experiment results of four different configurations when the extreme noise is a Pareto mixture. The configures of upper left, upper right, lower left and lower right are $ \alpha, \beta, d_u = (1, 1.5, 10), (1, 1.5, 5), (1, 2.5,  5)$ and $ (2, 2.5, 5)$ respectively. The results are averages of 50 repeated experiments.}
  \label{fig:synthetic2}
\end{figure}

\subsection{Semi-synthetic Dataset}\label{subsec:semisyn}

Now, we use the wavesurge dataset \cite{coles2001introduction} to create a semi-synthetic dataset for our experiments. The wavesurge dataset has 2894 data points, which contain wave and surge heights at a single location off south-west England. Since wave and surge heights are not in the same scale and may not be positive, we shift the data and normalize each dimension by its $ 10 \% $ quantile. Given the wavesurget dataset, we generate our semi-synthetic dataset in the following way. 
\begin{align}
    & X \sim \text{Unif}(0,1), D \sim \text{Ber}(p(X)), \text{where} \  p(x) = 1/(1 + \exp(-x^\mathsf{T} b)), \notag\\
    & Y = (1 -  X + D)W^{\alpha_1} S^{\alpha_2} +  N(0,1), \label{eq:semi_syn}
\end{align}
where $ W $ and $ S $ are the height of the wave and surge, respectively. In this experiment, we evaluate how well our proposed EVT‑based estimators recover the Normalized Extreme Treatment Effect (NETE) when only limited “short‑term” data are available. We split the dataset into a training set (1,000 observations) and a test set (1,894 observations). First, we estimate the NETE on the training set using four estimators. Next, we apply the identification formula from Proposition \ref{prop:id} together with (\ref{eq:semi_syn}) to obtain a high‑fidelity estimate of the NETE on the test set. Because the test‑set estimate leverages additional data and the correct tail model, we treat it as a surrogate “ground truth” for comparison.  The real-world implications of this experiment is that we can use some short-term data (the training set) to predict long-term and unobserved behavior (the test set). 

\cref{tab:semi-synthetic} shows the results we get using different estimators. The results show that our EVT-IPW and EVT-DR give estimations that are closer to the test-set estimate than the naive estimators. In particular, the naive estimators consistently overshoot the true NETE by an order of magnitude. In addition, while more extreme tail configurations (e.g.\ $(1,3)$) slightly increase variance, the EVT‑based methods remain stable, with EVT‑DR deviating by at most 0.3 from the test‑set estimate.  These findings demonstrate that incorporating multivariate extreme value structure via our EVT‑IPW and EVT‑DR estimators substantially improves finite‑sample estimation of treatment effects on rare, tail events, compared both to naive methods. 

\begin{table}[!ht]
  \centering
  \caption{Causal Effect Estimates}
  \label{tab:semi-synthetic}
  \begin{tabular}{@{}cccccc@{}}
    \toprule
    $(\alpha_1, \alpha_2) $ & EVT-DR & EVT-IPW & Naive-DR  & Naive-IPW&  Test‐set Estimate \\
    \midrule
     (2, 2) & 0.18 &  0.25& 41.93  & 27.34  & 0.13 \\ 
     (1, 3) & 0.43 & 0.44 & 17.04 & 15.68  &0.13 \\
     (2.5, 1) & 0.13 & 0.18  & 31.64 & 26.06 & 0.20\\
     (1.5, 1.5) & 0.26 & 0.23 & 7.91 & 9.46  & 0.20\\
    \bottomrule
  \end{tabular}
\end{table}
 
\section{Conclusion}
In this paper, we addressed the challenge of estimating treatment effects on rare, high‐impact events by uniting tools from causal inference and extreme value theory. We introduced a new estimand that explicitly captures how an intervention shifts the tail average of the outcome distribution. Exploiting the spectral–magnitude decomposition inherent in multivariate regular variation, we obtained a simple, implementable identification formula (Proposition \ref{prop:id}). Building on this, we constructed both inverse‐propensity‐weighted (IPW) and doubly‐robust (DR) estimators, and we further established non‐asymptotic error bounds under a Pareto‐type tail assumption (\cref{thm:rate}, \cref{cor:rate}). In simulations and real‐data experiments, our methods consistently outperformed naive estimators when targeting extreme outcomes, thereby validating our theoretical guarantees. We believe this work opens the door to more refined causal analyses in applications—such as disaster risk reduction and financial risk management—where understanding treatment effects in the tail is paramount. One limitation of this work is that we mainly use heuristics to estimate the scaling exponential $\alpha $, which lacks a theoretical guarantee. For future direction, we would like to explore how to estimate the exponential $\alpha$ and develop a more adaptive method for choosing thresholds for our estimators. 

\bibliography{reference}
\bibliographystyle{plainnat}

\newpage

\appendix

\section{Proofs}
\subsection{Identification Formula}
In this subsection, we derive the identification formula in Proposition
\ref{prop:id}.

\begin{proof}
We first prove (\ref{eq:naive_id2}). By Assumption \ref{assump:f_homo},
\begin{align*}
  \lim_{t \rightarrow \infty} \mathbb{E} &\left[ \frac{Y (1) - Y
  (0)}{t^{\alpha}} \mid \|U\|> t \right] \\
  & = \lim_{t \rightarrow \infty}
  \mathbb{E} \left[ \frac{f (X, 1, U) - f (X, 0, U)}{t^{\alpha}} \mid \|U\|> t
  \right] \\
  & = \lim_{t \rightarrow \infty} \mathbb{E} \left[ \frac{f (X, 1, U) - f (X,
  0, U)}{\|U\|^{\alpha}} \cdot \left( \frac{\|U\|}{t} \right)^{\alpha} \mid
  \|U\|> t \right]\\
  & = \lim_{t \rightarrow \infty} \mathbb{E} \left[ (g (X, 1, U / \| U \|) -
  g (X, 0, U / \| U \|) + 2 e (t)) \cdot \left( \frac{\|U\|}{t}
  \right)^{\alpha} \mid \|U\|> t \right]
\end{align*}

We next argue that $\lim_{t \rightarrow \infty} \mathbb{E} [(\| U \| /
t)^{\alpha} \mid \| U \| > t] = \alpha / (\beta - \alpha)$. We have
\begin{align*}
  \mathbb{E} [(\| U \| / t)^{\alpha} \mid \| U \| > t] & = 1 + \int^{\infty}_1
  P ((\| U \| / t)^{\alpha} \geqslant r \mid \| U \| > t) \tmop{dr}\\
  & = 1 + \frac{\int_1^{\infty} P (\| U \| / t > r^{1 / \alpha}) \tmop{dr}}{P
  (\| U \| > t)}\\
  & = 1 + \frac{\int_1^{\infty} \alpha P (\| U \| > {rt}) r^{\alpha - 1}
  \tmop{dr}}{P (\| U \| > t)}
\end{align*}

Note that $\| U \|$ is also regularly varying. By Potter's theorem \cite[Theorem 1.56]{bingham1989regular}, for any $\epsilon >0$ and sufficiently large $t$, we have
\begin{align*}
  \frac{P (\| U \| > {rt})}{P (\| U \| > t)} & \leqslant 2 r^{-\beta + \epsilon} .
\end{align*}

Take $\epsilon > 0$ such that $\alpha -\beta + \epsilon - 1 < - 1$, we have for sufficiently large $t$, 
\begin{align*}
  \frac{\int_1^{\infty} \alpha P (\| U \| > {rt}) r^{\alpha - 1}
  \tmop{dr}}{P (\| U \| > t)} & \leqslant 2\int^{\infty}_1 \alpha r^{\alpha -\beta + \epsilon - 1} \tmop{dr} < \infty .
\end{align*}

Therefore, by the dominance
convergence theorem,
\begin{align*}
  \mathbb{E} [(\| U \| / t)^{\alpha} \mid \| U \| > t] & \rightarrow 1 +
  \int_1^{\infty} \alpha r^{- \beta} \cdot r^{\alpha - 1} \tmop{dr} = \beta /
  (\beta - \alpha),
\end{align*}

which implies
\begin{align*}
  \lim_{t \rightarrow \infty} e (t) \mathbb{E} [(\| U \| / t)^{\alpha} \mid \|
  U \| > t] & = 0
\end{align*}

and (\ref{eq:naive_id2}) holds. We then verify the uniform integrability of
function
\[ h (U) =\mathbb{E}_X [(g (X, 1, U / \| U \|) - g (X, 0, U / \| U \|)) (\| U
   \| / t)^{\alpha}] . \]
Note that by Assumption \ref{assump:f_homo}, $g$ is a continuous function on a
compact set and thus is bounded by some constant $C > 0$. We have
\begin{align*}
  \mathbb{E} [| h (U) | \mid \| U \| > t] & \leqslant 2 C\mathbb{E} [(\|
  U \| / t)^{\alpha} \mid \| U \| > t] .
\end{align*}

We have proven that the Right Hand Side (RHS) converges to a constant as $t
\rightarrow \infty$, which implies that $\mathbb{E} [| h (U) | \mid \| U \| >
t]$ is uniformly bounded. We conclude that $h (U)$ is uniformly integrable. By the uniform‐integrability convergence theorem,
\begin{align*}
  \lim_{t \rightarrow \infty} \mathbb{E} &\left[ (g (X, 1, U / \| U \|) - g (X,
  0, U / \| U \|)) \cdot \left( \frac{\|U\|}{t} \right)^{\alpha} \mid \|U\|> t
  \right] \\
  & =\mathbb{E}_{(r, \theta) \sim \mathcal{L}} [(g (X, 1, \theta) - g
  (X, 0, \theta)) r^{\alpha}]\\
  & =\mathbb{E}_{\theta \sim \mathcal{L}} [(g (X, 1, \theta) - g (X, 0, \theta))] \mathbb{E}_{r \sim \mathcal{L}}
  [r^{\alpha}] .
\end{align*}

where $\mathcal{L} $ is the limiting distribution of $(\| U \| / t, U / \| U
\|) \mid \| U \| > t$ and we use the asymptotic independent property of regularly varying distributions (Definition \ref{def:mvarying}).
\end{proof}

\subsection{Non-asymptotic Analysis}

To obtain a convergence rate for the estimator $\hat{\theta}_n^t$, we first analyze the rate of the two factors $\hat{\gamma}_n$ and $\hat{\eta}_{n,t}^{\text{DR}}$.

\begin{lemma}
  \label{lemma:mu_rate}Under Assumption \ref{assump:iid}, \ref{assump:mvary}
  with probability at least $1 - \delta$, for sufficiently large $n$, we have
  \begin{align*}
    | \gamma - \hat{\gamma}_n | & \leqslant O \left( \left( \frac{\log (2 /
    \delta)}{n} \right)^{1 / (2 + \beta)} \right),
  \end{align*}
  
  where $\gamma = 1 / \beta$ is the EVI of $U$. 
\end{lemma}
\begin{proof}
  We adopt the non-asymptotic analysis of the adaptive Hill estimator for EVI
  in {\citep{boucheron2015tail}}. In the paper, the author adopts an adaptive
  estimator, choosing $k$ to be
  \begin{align*}
    k & = \max \left\{ k \in \{l_n, \cdots, n\} \text{and} \  \forall i \in
    \{l_n, \cdots, n\}, | \hat{\gamma}(i) - \hat{\gamma}(k) | \leqslant
    \frac{\hat{\gamma}(i) r_n (\delta)}{\sqrt{i}} \right\}
  \end{align*}
  where $\hat{\gamma}(i) = \frac{1}{i} \sum_{j=1}^i \log \frac{\|U_{(j)}\|}{\|U_{(i+1)}\|}$ and $r_n (\delta)$ scales like $\sqrt{\log ((2 / \delta) \log (n))}$.
  First, we verify the von Mises conditions in \cite{boucheron2015tail} under Assumption \ref{assump:mvary}. Let $F$ be the CDF of $\|U\|$. By
  Assumption~\ref{assump:mvary}, we know that
  \begin{align*}
    g (z) & = c \alpha^m  \prod_{i = 1}^m (1 + z_i)^{- \alpha - 1}, \|z\|_1 >
    \zeta k^{(1 - 2 s) / \alpha} .
  \end{align*}
  
  Note that
  \begin{align*}
    |c - 1| = \left| \frac{g (z) - \alpha^m  \prod_{i = 1}^m (1 + z_i)^{-
    \alpha - 1}}{\alpha^m  \prod_{i = 1}^m (1 + z_i)^{- \alpha - 1}} \right| &
    \leqslant \xi k^{- s} \leqslant \xi .
  \end{align*}
  
  Let $\tilde{Z}_1, \cdots, \tilde{Z}_m \sim \alpha c^{1 / m} (1 + z)^{-
  \alpha - 1}, z \geqslant c^{1 / (m \alpha)} - 1$. Then, we verify the upper bound for the von Mises function, i.e., $\sup_{s\geqslant t} |\eta (s)| \leqslant O( t^\rho)$ for some $\rho < 0$, where $\eta$ is the von Mises function.
  \begin{align}
    \eta (t) & = \frac{tU' (t)}{U (t)} - \frac{1}{\beta} \\
    & = \frac{1}{tU (t) f (U (t))} - \frac{1}{\beta},  \label{eq:vonmise_fun}
  \end{align}
  
  where $f (t)$ is the density function of $\sum_{i = 1}^m a_i \tilde{z}_i$. By
  {\citetext{\citealp{nguyen2014tail},  Theorem 2.1}}, we have that when
  $\|U\|_1 > \max_i \{a_i \} \zeta k^{(1 - 2 s) / \alpha}$,
  \begin{align}
    f (t) & = C \beta t^{- \beta - 1} (1 + D (1 - 1 / \beta) t^{- 1} + o (t^{-
    1})) .  \label{eq:density_f}
  \end{align}
  
  Then,
  \begin{align*}
    \bar{F} (t) = 1 - F(t) & = Ct^{- \beta} (1 + Dt^{- 1} + o (t^{- 1}))
  \end{align*}
  
  and
  \begin{align}
    U (t) & = C^{1 / \beta} t^{1 / \beta} (1 + DC^{- 1 / \beta} t^{- 1 /
    \beta} / \beta + o (t^{- 1 / \beta})) .  \label{eq:u_fun}
  \end{align}
  
  Plug in (\ref{eq:density_f}) and (\ref{eq:u_fun}) into
  (\ref{eq:vonmise_fun}), we get
  \begin{align*}
    \eta (t) & = \frac{1}{\beta (1 - DC^{- 1 / \beta} t^{- 1 / \beta} + o
    (t^{1 / \beta}))} - \frac{1}{\beta} = O (t^{- 1 / \beta}) .
  \end{align*}
  
  Therefore, the growth rate of the von Mises function is bounded. By
  {\citep{boucheron2015tail}}, with probability at least $1 - \delta$, we have
  \begin{align*}
    | \gamma - \hat{\gamma}_n | & \leqslant O \left( \left( \frac{\log (2 /
    \delta)}{n} \right)^{1 / (2 + \beta)} \right),
  \end{align*}
\end{proof}

\begin{lemma}
  \label{lemma:eta_rate}Undet the assumption of Theorem \ref{thm:rate}, with
  probability at least $1 - \delta$, we have
  \begin{align*}
    | \hat{\eta}_{n, t}^{\tmop{DR}} - \eta | & \leqslant O (\sqrt{R_p (n/2, \delta) R_g
    (t^{- \beta} n, \delta)} + t^{\beta / 2} n^{- 1 / 2} + t^{- \min \{1,
    \beta\}} + t^{- \beta s / (1 - 2 s)} + e (t) + \log (t) R_{\alpha} (n,\delta)) .
  \end{align*}
\end{lemma}
\begin{proof}
  Let
  \begin{align*}
    \eta^t & =\mathbb{E} \left[ \frac{f (X, 1, U) - f (X, 0,
    U)}{\|U\|^{\hat{\alpha}_{n}}} \mid \|U\|> t \right],
  \end{align*}
  
  We have the following decomposition
  \begin{align*}
    | \hat{\eta}_{n, t}^{\tmop{DR}} - \eta | & \leqslant | \hat{\eta}_{n,
    t}^{\tmop{DR}} - \eta^t | + | \eta^t - \eta |.
  \end{align*}
  
  The first term comes from the standard statistical error of DR estimator, while the second term is the bias term caused by the finite threshold. For the first term, by standard DML theory {\citep{foster2023orthogonal}},
  we have
  \begin{align*}
    | \hat{\eta}_{n, t}^{\tmop{DR}} - \eta^t | & \leqslant O \left( \sqrt{R_p
    (n/2, \delta) R_g (n_t, \delta)} + n_t^{- 1 / 2} \right),
  \end{align*}
  
  where $n_t = \sum_{i = 1}^{n/2} I (\|U_i \|> t)$ is a random variable. By
  Bernstein's inequality, with probability at least $1 - \delta$,
  \begin{align*}
    n_t - n\mathbb{P}(\|U\|> t)/2 & \geqslant O \left( \log (1 / \delta) +
    \sqrt{n\mathbb{P}(\|U\|> t) \log (1 / \delta)} \right)
  \end{align*}
  
  Therefore, with the same probability, when $n \geqslant \Theta (\log (1 /
  \delta) t^{\beta})$.
  \begin{align*}
    n_t & \geqslant \frac{1}{4} n\mathbb{P}(\|U\|> t) = \Theta (nt^{- \beta})
  \end{align*}
  
  and we have
  \begin{align}
    | \hat{\eta}_{n, t}^{\tmop{DR}} - \eta^t | & \leqslant O \left( \sqrt{R_p
    (n/2, \delta) R_g (nt^{- \beta}, \delta)} + t^{\beta / 2} n^{- 1 / 2}
    \right),  \label{eq:eta_est}
  \end{align}
  
  where we use the monotonicity of $R_p, R_g$.
  
  For the second term (the bias term),
  \begin{align*}
    | \eta^t - \eta | & = \left| \mathbb{E} \left[ \frac{f (X, 1, U) - f (X,
    0, U)}{\|U\|^{\hat{\alpha}_{n}}} \mid \|U\|> t \right] -\mathbb{E}[g (X, 1, U
    /\|U\|) - g (X, 0, U /\|U\|)] \right|\\
    & \leqslant \left| \mathbb{E} \left[ \frac{f (X, 1, U) - f (X, 0,
    U)}{\|U\|^{\alpha}} - g (X, 1, U /\|U\|) - g (X, 0, U /\|U\|) \mid \|U\|>
    t \right] \right|\\
    & \quad \quad + |\mathbb{E}[g (X, 1, U /\|U\|) - g (X, 0, U /\|U\|)]
    -\mathbb{E}[g (X, 1, U /\|U\|) - g (X, 0, U /\|U\|) \mid \|U\|> t] |\\
    & \qquad + \left| \mathbb{E} \left[ \frac{f (X, 1, U) - f (X, 0,
    U)}{\|U\|^{\alpha}} \left( 1 - {\| U \|^{\alpha - \hat{\alpha}_{n}}}
    \right) \right] \right|\\
    & \leqslant 2 e (t) + |\mathbb{E}[g (X, 1, U /\|U\|) - g (X, 0, U
    /\|U\|)] -\mathbb{E}[g (X, 1, U /\|U\|) - g (X, 0, U /\|U\|) \mid \|U\|>
    t] |,\\
    & \qquad + \left| \mathbb{E} \left[ \frac{f (X, 1, U) - f (X, 0,
    U)}{\|U\|^{\alpha}} \left( 1 - {\| U \|^{\alpha - \hat{\alpha}_{n}}}
    \right) \right] \mid \|U\|> t \right|
  \end{align*}
  
  where we use Assumption \ref{assump:f_homo} in the last equality. By the
  error rate assumption in Theorem \ref{thm:rate},
  \begin{align*}
    \left| \mathbb{E} \left[ \frac{f (X, 1, U) - f (X, 0, U)}{\|U\|^{\alpha}}
    \left( 1 - {\| U \|^{\alpha - \hat{\alpha}_{n}}} \right) \right] \mid
    \|U\|> t \right| & \leqslant C\mathbb{E} \left[ \left| 1 - {\| U\|^{\alpha - \hat{\alpha}_{n}}} \right| \mid \| U \| > t \right]\\
    & = O (\log (t) R_{\alpha} (n, \delta)) .
  \end{align*}
  
  Since $g$ is $L$-Lipschitz continuous, the second term is upper bounded by
  Wasserstein distance $LW_1 (\mathcal{L}_{U /\|U\|}^t, \mathcal{L}_{U
  /\|U\|})$, where $\mathcal{L}_{U /\|U\|}^t$ is the distribution of $U
  /\|U\|$ conditioning on $\|U\|> t$ and $\mathcal{L}_{U /\|U\|}$ is its
  limiting distribution as $t \rightarrow \infty$. Therefore, we have
  \begin{align}
    | \eta^t - \eta | & \leqslant 2 e (t) + LW_1 (\mathcal{L}_{U /\|U\|}^t,
    \mathcal{L}_{U /\|U\|}) + O (\log (t) R_{\alpha} (n, \delta)) \nonumber\\
    & \leqslant 2 e (t) + O (t^{- \min \{1, \beta\}} + t^{- \beta s / (1 - 2
    s)} + \log (t) R_{\alpha} (n, \delta)),  \label{eq:eta_bias}
  \end{align}
  
  where we use {\citetext{\citealp{zhang2023wasserstein},  Proposition 3.1}}
  in the last inequality to upper bound the bias term $W_1 (\mathcal{L}_{U
  /\|U\|}^t, \mathcal{L}_{U /\|U\|})$. Combing (\ref{eq:eta_est}) and
  (\ref{eq:eta_bias}), we get
  \begin{align*}
    | \hat{\eta}_n^{t} - \eta | & \leqslant O (\sqrt{R_p (n/2, \delta) R_g
    (t^{- \beta} n, \delta)} + t^{\beta / 2} n^{- 1 / 2} + t^{- \min \{1,
    \beta\}} + t^{- \beta s / (1 - 2 s)} + e (t) + \log (t) R_{\alpha} (n,\delta)) .
  \end{align*}
\end{proof}

\begin{lemma}\label{lemma:eta_ipw}
  Under the assumption of Theorem \ref{thm:rate}, with probability at least $1
  - \delta$, we have
\begin{align*}
    | \hat{\eta}_{n, t}^{\text{IPW}} - \eta | & \leqslant O (R_p (n/2, \delta) +
    t^{\beta / 2} n^{- 1 / 2} + t^{- \min \{1, \beta\}} + t^{- \beta s / (1 -
    2 s)}) + e (t) .
  \end{align*}
\end{lemma}
\begin{proof}

Similar to Lemma \ref{lemma:eta_rate}, we have the following decomposition.
\begin{align*}
  | \hat{\eta}_{n, t}^{\text{IPW}} - \eta | & \leqslant | \hat{\eta}_{n,
  t}^{\text{IPW}} - \eta^t | + | \eta^t - \eta |.
\end{align*}

Term $| \eta^t - \eta |$ can be bounded in the same way as in the proof of
Lemma \ref{lemma:eta_rate}. By \cite[Theorem 1]{su2023estimated}, we have
\begin{align*}
  | \hat{\eta}_{n, t}^{\text{IPW}} - \eta^t | & \leqslant O (R_p (n/2, \delta) +
  n_t^{- 1 / 2}) = O (R_p (n/2, \delta) + t^{\beta / 2} n^{- 1 / 2}  + \log (t) R_{\alpha} (n,\delta)) .
\end{align*}
The rest of the proof is similar.   
\end{proof}

Now we are ready to prove Theorem \ref{thm:rate}.
\begin{proof}[Proof of \cref{thm:rate}]
  Note that by the asymptotic independence property of regularly varying
distribution,
\begin{align}
  | \hat{\theta}_{n, t}^{\tmop{DR}} - \theta^{\tmop{NETE}} | & = |
  \hat{\eta}_{n, t}^{\tmop{DR}} \cdot \hat{\mu}_n - \eta \cdot \mu | \notag \\
  & \leqslant | \mu | \cdot | \hat{\eta}_{n, t}^{\tmop{DR}} - \eta | + |
  \hat{\eta}_{n, t}^{\tmop{DR}} | \cdot | \hat{\mu}_n - \mu |. \label{eq:decomp}
\end{align}

By Lemma \ref{lemma:mu_rate} and \ref{lemma:eta_rate}, with high probability, $ \hat{\gamma}_n $ and $| \hat{\eta}_{n,t}^{\tmop{DR}} |$ is bounded. Note that by Lemma \ref{lemma:mu_rate},
\begin{align*}
  | \hat{\mu}_n - \mu | & = \left| \frac{1}{1 - \hat{\alpha}_{n} \hat{\gamma}_n} -
  \frac{1}{1 - \alpha \gamma} \right| = O (| \hat{\alpha}_{n} - \alpha | + |
  \hat{\gamma}_n - \gamma |) = O (\log (1 / \delta) n^{- 1 / (2 + \beta)} +
  R_{\alpha} (n, \delta))
\end{align*}

Therefore, by Lemma \ref{lemma:eta_rate} and (\ref{eq:decomp}),
\begin{align*}
  | \hat{\theta}_{n, t}^{\tmop{DR}} - \theta^{\tmop{NETE}} | & \leqslant O
  ( \sqrt{R_p (n/2, \delta) R_g (nt^{- \beta}, \delta)} + t^{\beta / 2}n^{- 1 / 2} + \log (1 / \delta) n^{- 1 / (2 + \beta)} \\
  &\quad\quad\quad\quad + t^{- \min \{1,\beta\}} + t^{- \beta s / (1 - 2 s)} + \log (t) R_{\alpha} (n, \delta) + e(t)) .
\end{align*}

The bound for $\hat{\theta}_{n, t}^{\tmop{IPW}}$ can be proven similarly.
  \end{proof}

\begin{corollary}[Convergence rate for IPW]
  \label{cor:rate_ipw}Under the assumptions of Theorem~\ref{thm:rate}, further
  suppose that
  \begin{align*}
    R_p (n, \delta) = \Theta (\log (1 / \delta) n^{- 1 / 2}) , R_g (n, \delta) = \Theta (\log (1 / \delta) n^{- 1 / 2}), R_\alpha (n, \delta) = \Theta( \log(1/\delta) n^{-c_\alpha}),
  \end{align*}
  for some $c_\alpha > 0 $, the following conclusions hold.
  \begin{enumerate}
    \item If $s \in (0, 1 / (2 + \max \{1, \beta\}))$, takes $t_n = \Theta (n^{
    (1 - 2 s) \hat{\gamma}_{n}})$, with probability at least $1 - \delta$,
    we have
    \begin{align*}
      | {{\widehat{\theta }_{n, t}^{\text{IPW} }} }   - \theta^{NETE} | & = O
      (e (t_n) + n^{- s} \log (1 / \delta) +  n^{-c_\alpha} \log(n)\log (1 / \delta)) .
    \end{align*}
    \item If $s \in [1 / (2 + \max \{1, \beta\}), 1 / 2)$, \ takes $t = \Theta
    (n^{(\hat{\gamma}_{n} / (1 + 2 \min \{1, \hat{\gamma}_{n} \})})$,
    with probability at least $1 - \delta$, we have
    \begin{align*}
      | {{\widehat{\theta }_{n, t}^{\text{IPW} }} }   - \theta^{NETE} | & = O
      (e (t_n) + n^{- 1 / (2 + \max \{\beta, 1\})} \log (1 / \delta) +  n^{-c_\alpha} \log(n)\log (1 / \delta)) .
    \end{align*}
  \end{enumerate}
\end{corollary}

\begin{proof}[Proof of \cref{cor:rate}]
    By \cref{thm:rate} and the error rate assumption in \cref{cor:rate},  we have
    \begin{align*}
    \left| {{\widehat{\theta }_{n, t}^{\text{DR} }} }   -
    \theta^{NETE} \right|  & \leqslant O (\log(1/\delta) t^{\beta/4}n^{-1/2} + t^{\beta / 2} n^{-1 / 2}+ \log (1 / \delta) n^{- 1 / (2 +\beta)}  \notag\\
    & \quad\quad \quad \quad + t^{- \min \{1, \beta\}}  + t^{- \beta s / (1 - 2 s)} + \log(t)\log(1/\delta) n^{-c_\alpha}+ e (t)) \\
    &\leqslant O (\log(1/\delta) t^{\beta/2}n^{-1/2}+ \log (1 / \delta) n^{- 1 / (2 +\beta)} \\
    & \quad\quad \quad \quad + t^{- \min \{1, \beta\}}  + t^{- \beta s / (1 - 2 s)} + \log(t)\log(1/\delta) n^{-c_\alpha}+ e (t)).
    \end{align*}
    If $s \in (0, 1 / (2 + \max \{1, \beta\}))$, we have 
    \begin{align*}
        \left| {{\widehat{\theta }_{n, t}^{\text{DR} }} }   -
    \theta^{NETE} \right| &\leqslant O (\log(1/\delta) t^{\beta/2}n^{-1/2}+ \log (1 / \delta) n^{- 1 / (2 +\beta)} \\
    &\quad\quad \quad \quad + t^{- \beta s / (1 - 2 s)} + \log(t)\log(1/\delta) n^{-c_\alpha}+ e (t)).
    \end{align*}
    Takes $t_n = \Theta(n^{(1-2s)\hat{\gamma}_n})$, we get
    \begin{align}
        \left| {{\widehat{\theta }_{n, t_n}^{\text{DR} }} }   -
    \theta^{NETE} \right| &\leqslant O (\log(1/\delta) n^{{(1-2s)\hat{\gamma}_n}\beta /2 -1/2}+ \log (1 / \delta) n^{- 1 / (2 +\beta)} \notag\\
    &\quad\quad\quad\quad+ n^{- \hat{\gamma}_n\beta s} + \log(t)\log(1/\delta) n^{-c_\alpha}+ e (t_n)).\label{eq:small_s}
    \end{align}
    By \cref{lemma:mu_rate}, we have 
    \begin{align*}
        |\hat{\gamma}_n - \gamma| = |\hat{\gamma}_n - 1/\beta| \leqslant O(\log(1/\delta)n^{-1/(2+\beta)}). 
    \end{align*}
    Therefore, $ n^{\hat{\gamma}_n \beta} = 1 + O( n^{-1/(2+\beta)}) $. Plug this bound into (\ref{eq:small_s}) and we can get the results. Similarly, if $s \in [ 1 / (2 + \max \{1, \beta\}), 1/2)$, we have 
    \begin{align*}
        \left| {{\widehat{\theta }_{n, t}^{\text{DR} }} }   -
    \theta^{NETE} \right| &\leqslant O (\log(1/\delta) t^{\beta/2}n^{-1/2}+ \log (1 / \delta) n^{- 1 / (2 +\beta)} \\
    &\quad\quad \quad \quad + t^{- \min\{1,\beta\}} + \log(t)\log(1/\delta) n^{-c_\alpha}+ e (t)).
    \end{align*}
    Take $t_n = \Theta(\hat{\gamma}_n / (1 + 2\min\{1, \hat{\gamma}_n\}))$, we get the results. 
\end{proof}
  
\section{Experiment Details}
In this section, we introduce some details in our experiments. The two baseline estimators we consider are naive-IPW: 
\begin{align*}
    \widehat{\theta}_{n,t}^{\text{Naive-IPW}} = \frac{1}{t^\alpha n_t}\sum_{i>n/2: \|U_i\|\geqslant t} Y_i \Bigl(\frac{D_i}{\hat p(X_i)} - \frac{1-D_i}{1-\hat p(X_i)}\Bigr) .\label{eq:eta_ipw}
\end{align*}
and naive-DR: 
\begin{align*}
    \hat\eta_{n,t}^{\mathrm{Naive-DR}} = \frac{1}{t^\alpha n_t} \sum_{i> n/2: \|U_i\|\geqslant t}\Bigl[\hat g(X_i,1,U_i)-\hat g(X_i,0,U_i) + \frac{D_i-\hat p(X_i)}{\hat p(X_i)(1-\hat p(X_i))}\bigl(Y_i-\hat g(X_i,D_i,U_i)\bigr)\Bigr],
\end{align*}
where $n_t = \sum_{i= \lfloor n/2 \rfloor + 1}^{n} I(\|U_i\|\geqslant t)$ and  $\hat{p}(X)$ and $ \hat{g}(\cdot)$ are the estimated propensity function and the outcome function respectively. The nuisance estimation of $\hat{g}$ is obtained by running a regression $Y\sim (X,D,U)$. We clip the propensity to $[10^{-4}, 1 -10^{-4}]$ to ensure the overlap assumption (\cref{assump:overlap}). $\epsilon \sim \text{Unif} ([-1,1])$ in the data generation in synthetic experiments. We use sample splitting in our experiment, using the first half for nuisance estimation. In the experiment, we use the same threshold $t$ for all estimators, which is given by \cref{cor:rate}. To choose the threshold, we first use the adaptive Hill estimator \cite{boucheron2015tail} to get an estimation of EVI $\hat{\gamma}_{n}$ and then set the threshold to be $t = 0.25n^{(\hat{\gamma}_{n} / (1 + 2 \min \{1, \hat{\gamma}_{n} \})}$ as in \cref{thm:rate}. The approximate exponential $ \hat{\alpha}_n $ is coefficient of $\log(\|U\|)$ in linear regression $\log(|Y|) \sim \log(\|U\|)$. For the adaptive Hill estimator \cite{boucheron2015tail}, we follow authors' choice for hyperparameters and choose $l_n = 30, r(\delta) = \sqrt{\log\log(n)}$ and 
\begin{equation}
           k  = \min \left\{ k \in \{l_n, \cdots, n\} \  \text{and} \  \exists \ i \in \{l_n,\cdots, n\}, | \hat{\gamma}(i) - \hat{\gamma}(k) | > \frac{\hat{\gamma}(i) r_n (\delta)}{\sqrt{i}} \right\} - 1, \notag
\end{equation}
where $\hat{\gamma}(i) = \frac{1}{i} \sum_{j=1}^{i} \log\frac{\|U_{(j)}\|}{\|U_{(i+1)}\|}$.

We run logistic regression to estimate the propensity function and use random forest to model the outcome. 

For the semi-synthetic experiment, we apply the same hyperparameter as above to estimate NETE. We shift the data to make it positive and normalize each dimension by its 10 \% quantile. The \cref{fig:scatter_fig} shows the rough distribution of the wavesurge data after these transformations. 
\begin{figure}[!htb]
    \centering
    \includegraphics[width=0.5\linewidth]{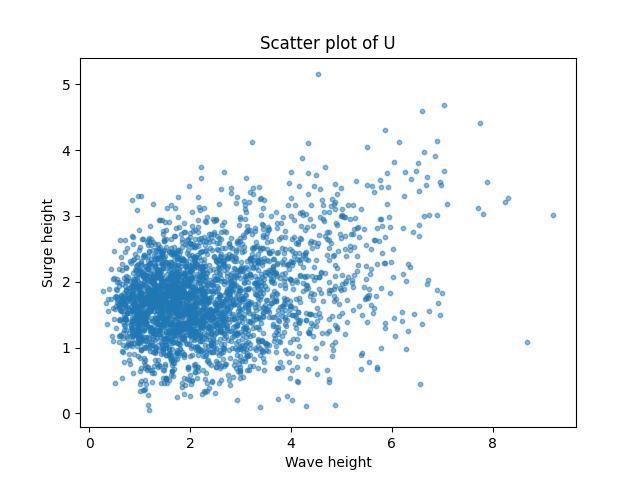}
    \caption{The scatter plot of wavesurge data.}
    \label{fig:scatter_fig}
\end{figure}

We now describe how we calculate test-set estimation in our experiments. By the data generation process, 
\begin{align*}
    \theta^{\text{NETE}} &= \lim_{t\rightarrow\infty}\mathbb{E}[ \frac{W^{\alpha_1}S^{\alpha_2}}{t^{\alpha_1 + \alpha_2}} \mid \|U\| > t] \\
    &= \lim_{t\rightarrow\infty}\mathbb{E}[ \frac{W^{\alpha_1}S^{\alpha_2}}{\|U\|^{\alpha_1 + \alpha_2}} \mid \|U\| > t] \cdot \frac{1}{1 - (\alpha_1 + \alpha_2)\gamma},
\end{align*}
where we use \cref{prop:id} in the second equality. We know the ground-truth $\alpha_1, \alpha_2$ and we can estimate the EVI $\gamma$ using the test set. Suppose the estimated EVI is $\hat{\gamma}$, we set the threshold to $ t_n = 0.25n^{(\hat{\gamma} / (1 + 2 \min \{1, \hat{\gamma} \})} $ and get estimation 

\begin{equation*}
    \hat{\theta}_{\text{test}}^{\text{NETE}} = \mathbb{E}_n[ \frac{W^{\alpha_1}S^{\alpha_2}}{\|U\|^{\alpha_1 + \alpha_2}} \mid \|U\| > t_n] \cdot \frac{1}{1 - (\alpha_1 + \alpha_2)\hat{\gamma}}. 
\end{equation*}
\end{document}

%% file: macro.tex
\newcommand{\nobracket}{}

\newcommand{\tmop}[1]{\ensuremath{\operatorname{#1}}}

\newcommand{\tmstrong}[1]{\textbf{#1}}